\documentclass[conference]{IEEEtran}
\newcommand{\beq}{\begin{equation}}
\newcommand{\enq}{\end{equation}}
\newcommand{\beqa}{\begin{eqnarray}}
\newcommand{\enqa}{\end{eqnarray}}

\newcommand{\be}{\beta}

\newcommand{\qed}{\hfill $\Box$}

\newtheorem{theorem}{Theorem}
\newenvironment{proof}{{\sl Proof\/}:\ \ }{\qed\vspace{\baselineskip}}

\def\bbC{{\sf C}\kern -6pt {\sf C}}
\def\bbF{{\sf F}\kern -5pt {\sf F}}
\def\bbR{{\sf R}\kern -6pt {\sf R}}
\def\bbZ{{\sf Z}\kern -5pt {\sf Z}}
\def\sfbegin{\begingroup\sf}
\def\sfend{\endgroup}

\def\be{\begin{eqnarray*}}
\def\ee{\end{eqnarray*}}

\usepackage[final]{graphicx}
\usepackage[reqno]{amsmath}
\usepackage{amssymb}
\usepackage{wrapfig}
\usepackage{subfig}
\usepackage{epstopdf,comment,color}
\IEEEoverridecommandlockouts

\ifCLASSINFOpdf
\else
\fi
\begin{document}
%
\title{On Randomized Distributed Coordinate Descent with Quantized Updates}

\author{\IEEEauthorblockN{Mostafa El Gamal}
\IEEEauthorblockA{Dept. of Electrical and\\Computer Engineering\\
Worcester Polytechnic Institute\\
Worcester, Massachusetts\\
Email: melgamal@wpi.edu}
\and
\IEEEauthorblockN{Lifeng Lai}
\IEEEauthorblockA{Dept. of Electrical and\\Computer Engineering\\
University of California, Davis\\
Davis, California\\
Email: lflai@ucdavis.edu}}


%


\maketitle

\begin{abstract}
In this paper, we study the randomized distributed coordinate descent algorithm with quantized updates. In the literature, the iteration complexity of the randomized distributed coordinate descent algorithm has been characterized under the assumption that machines can exchange updates with an infinite precision. We consider a practical scenario in which the messages exchange occurs over channels with finite capacity, and hence the updates have to be quantized. We derive sufficient conditions on the quantization error such that the algorithm with quantized update still converge. We further verify our theoretical results by running an experiment, where we apply the algorithm with quantized updates to solve a linear regression problem.
\end{abstract}


\section{Introduction}\label{sec:intro}
Developing algorithms to solve distributed optimization problems has attracted significant interests in recent years. Distributed optimization problems naturally arise in various scenarios. For example, in machine learning problems, the training dataset might be too large to be stored in a single machine. Another example is when data is collected (and hence is naturally located) at multiple locations. Distributed algorithms are also useful to harness parallel processing capabilities of multiple machines.

In distributed optimization, it is essential for machines involved to exchange messages. As communication links between machines have limited capacity and have significantly longer delay, many recent papers focus on developing algorithms that are communication efficient. In \cite{martin}, an algorithm was proposed to reduce the amount of necessary communication by using the local computation in a primal-dual setting. Another communication efficient algorithm for empirical risk minimization was introduced in \cite{yuchen}. ADMM was considered in \cite{boyd,wei,hong} to handle the communication bottleneck.



Most of the existing studies analyze how many rounds of communications are required for the convergence of the developed algorithms. In each communication round, it is typically assumed that machines can exchange messages with an infinite precision. However, in practice, these data exchanges occur over physical channels that have limited capacity. As a result, machines cannot exchange messages with an infinite precision and need to quantize messages before sending them to other machines. A natural question to ask is whether these distributed algorithms will still converge if the exchanged messages are quantized. If these algorithms still converge, one can further ask what are the effects of the quantization on the converge speed.

In this paper, we answer these questions for a particular optimization algorithm, namely randomized coordinate descent \cite{wright}. This algorithm is easily implementable to solve distributed optimization problems since each machine can compute a single coordinate of the gradient. In each iteration of the randomized coordinate descent, the algorithm takes a step in the direction of a randomly chosen coordinate in order to decrease the function value. This is done by computing the partial derivatives, which is much cheaper computationally than taking a full gradient step. 
The iteration complexities of the randomized coordinate descent algorithms are analyzed in \cite{peter1,peter2} under a very general setup. In \cite{peter3}, a hybrid coordinate descent method (Hydra) was presented  to speed up the coordinate descent algorithm. Asynchronous parallel processing was analyzed in \cite{jordan} for a number of optimization algorithms including the randomized coordinate descent.

We answer the above questions by first modifying a distributed version of the coordinate descent algorithm to fit the paradigm of capacity limited communication. We then determine sufficient conditions on the quantization error such that the algorithm converges to the optimal solution. In particular, we apply our algorithm to an unconstrained minimization problem of a function $f$ that is $L$-smooth and $m$-strongly convex. We show that for an accuracy level $\epsilon$ and a confidence level $\rho$, our algorithm converges to the optimal solution if the quantization error $\Delta$ is upper bounded by a function of $\epsilon$, $\rho$, $L$, $m$, and $d$, where $d$ is the number of features. We verify the results by running an experiment, where we apply our algorithm to solve a linear regression problem.

The rest of the paper is organized as follows. We give a formal statement of the problem in Section~\ref{sec:formulation}. In Section~\ref{sec:algorithm} we introduce our algorithm. We analyze the convergence rate of our algorithm, and we derive sufficient conditions on the quantization error in Section~\ref{sec:convergence}. We verify our results by running an experiment in Section~\ref{sec:simulation}. Finally, we conclude the paper in Section~\ref{sec:conclusion}.
\section{Problem Formulation} \label{sec:formulation}
We consider an unconstrained convex minimization problem
\begin{eqnarray}
\min_{\mathbf x \in \mathbb R^d} f(\mathbf x),
\end{eqnarray}
where $\mathbf x=\{x_1,x_2,...,x_d\}$, and $f(\mathbf x)$ is an $L$-smooth and $m$-strongly convex function, such that for all $\mathbf x, \mathbf y \in \mathbb R^d$, we have that
\begin{eqnarray}
||\triangledown f(\mathbf x) - \triangledown f(\mathbf y)|| &\le& L ||\mathbf x-\mathbf y||, \label{eq:smooth} \\
\langle \triangledown f(\mathbf x) - \triangledown f(\mathbf y),\mathbf x-\mathbf y \rangle &\ge& m ||\mathbf x-\mathbf y||^2, \label{eq:convex}
\end{eqnarray}
where $L$ is the Lipschitz constant and $m$ is the strong convexity parameter. The condition number of $f$ is defined as $g=L/m$. As a result of the strong convexity, the function $f(\mathbf x)$ has a unique minimum at $\mathbf x^*$.

In the distributed coordinate descent algorithm, the data examples related to the problem are distributed over $d$ nodes such that each node can calculate one coordinate of the gradient $\triangledown f(\mathbf x)$ as explained in Section 6 of \cite{peter3}. The algorithm we study in this paper is the randomized coordinate descent, in which at each iteration a coordinate is randomly selected to be updated. There are different ways to randomly select the coordinate. In this paper, we focus on the case in which the coordinates are selected with a uniform distribution. The channels connecting machines are capacity limited with a quantization resolution of $\Delta$, which means that machine $i$ can only send a quantized version $Q\left(\frac {\partial f(\mathbf x)} {\partial x_{i}}\right)$ of its update $\frac {\partial f(\mathbf x)} {\partial x_{i}}$, such that
\begin{eqnarray} \label{eq:delta}
\resizebox{0.9\hsize}{!}{$Q\left(\frac {\partial f(\mathbf x)} {\partial x_{i}}\right)=y\Delta, ~\text{if}~(y-\frac 1 2)\Delta \le \frac {\partial f(\mathbf x)} {\partial x_{i}} < (y+\frac 1 2)\Delta$},
\end{eqnarray}
in which $Q(\cdot)$ is the quantization operator. Let $[\triangledown f(\mathbf x)]_{i} \in \mathbb R^d$ denote a vector that has only one nonzero element at position $i$ that is equal to $\frac {\partial f(\mathbf x)} {\partial x_{i}}$. By applying the quantization operator to the nonzero element of the vector $[\triangledown f(\mathbf x)]_{i} \in \mathbb R^d$, we can rewrite (\ref{eq:delta}) as
\begin{eqnarray} \label{eq:noise}
Q([\triangledown f(\mathbf x)]_{i})=[\triangledown f(\mathbf x)]_{i}-\mathbf n,
\end{eqnarray}
where $\mathbf n \in \mathbb R^d$ is the quantization noise vector. The noise vector $\mathbf n$ has only one nonzero element $n_i$ that is bounded as $|n_i| \le \Delta/2$. Hence,
\begin{eqnarray} \label{eq:quan}
||\mathbf n|| \le \frac \Delta 2.
\end{eqnarray}
Throughout the paper, we use $\mathbf x_k$ and $\mathbf x^q_k$ to denote the $k$th update of $\mathbf x$ before and after adding the quantization noise, respectively. An upper case letter $S$ is used for a random variable, while a lower case letter $s$ is used for a realization of $S$. We also use $||\mathbf x||$ to denote the Euclidean norm of the vector $\mathbf x$, and we use $Q(.)$ to denote the quantization operator.
\section{Quantized Randomized Coordinate Descent} \label{sec:algorithm}
Here, we describe the randomized coordinate descent algorithm with quantized update. The algorithm starts from an initial point $\mathbf x_0$, and stops after a predetermined number of iterations $T$. Set $\mathbf x_0^q=\mathbf x_0$. At iteration $(j+1)$, a machine $s_{j+1} \in \{1,2,...,d\}$ is randomly (with a uniform distribution) selected, who calculates $[\triangledown f(\mathbf x_j^q)]_{s_{j+1}}$ and then sends the quantized update $Q\left([\triangledown f(\mathbf x_j^q)]_{s_{j+1}}\right)$, all machines update
\begin{eqnarray}
\mathbf x_{j+1}^q=\mathbf x_j^q-tdQ([\triangledown f(\mathbf x_j^q)]_{s_{j+1}}),
\end{eqnarray}
where $t$ is the step size.

\begin{table}[htb]
\centering
\begin{tabular}{l}
{\bf Algorithm:} Quantized Randomized Coordinate Descent\\ 
\hline \hline
1: $\mathbf x_0^q=\mathbf x_0$\\
2: {\bf for} $j=0,1,...,(T-1)$ {\bf do}\\
3: a machine is randomly selected to send its update \\
4: selected machine $s_{j+1}$ computes $[\triangledown f(\mathbf x_j^q)]_{s_{j+1}}$ \\
5: machine $s_{j+1}$ communicates $Q([\triangledown f(\mathbf x_j^q)]_{s_{j+1}})$ \\
6. all machines update $\mathbf x_{j+1}^q=\mathbf x_j^q-tdQ([\triangledown f(\mathbf x_j^q)]_{s_{j+1}})$\\
7: {\bf end for}
\end{tabular}
\label{tb:algorithm}
\end{table}

To facilitate the analysis, we also record the sequence
\begin{eqnarray}
\mathbf x_{j+1}=\mathbf x_j^q-td [\triangledown f(\mathbf x_j^q)]_{s_{j+1}}.
\end{eqnarray}

Using (\ref{eq:noise}), we can show that
\begin{eqnarray}
\mathbf x_j^q=\mathbf x_j+td \mathbf n_j,~ j=\{1,2,...,T\}.
\end{eqnarray}

It is desirable that the algorithm converges within $k$ iterations to an accuracy level of $\epsilon$ and a confidence level of $\rho \in (0,1)$, such that
\begin{eqnarray} \label{eq:confidence}
\text{Pr}(||\mathbf x_k-\mathbf x^*||^2 \le \epsilon) \ge 1-\rho.
\end{eqnarray}
By applying Markov inequality, the convergence condition in (\ref{eq:confidence}) is achieved if
\begin{eqnarray} \label{eq:convergence}
\mathbb E||\mathbf x_k-\mathbf x^*||^2 \le \epsilon \rho.
\end{eqnarray}

\section{Convergence Analysis} \label{sec:convergence}

In this section, we analyze the convergence rate of the quantized randomized coordinate descent algorithm.
\begin{theorem}
Given that the quantization error $\Delta$ is bounded as following
\begin{eqnarray} 
\Delta \le \frac {\epsilon \rho L^2} {2m}(\frac 1 C_{min}-1), \nonumber
\end{eqnarray}
the number of iterations required for the quantized randomized coordinate descent algorithm to converge to the optimal solution $\mathbf x^*$ is at most
\begin{eqnarray}
k^q&=& \frac {\log (2||\mathbf x_0-\mathbf x^*||^2/\epsilon \rho)}{\log(1/C_{min})} \nonumber \\
&+& \frac {\log (2||\mathbf x_0-\mathbf x^*||^2)}{\log(1/(C_{min}+\frac {\epsilon \rho}{2}(1-C_{min}))}, \nonumber
\end{eqnarray}
where
\begin{eqnarray}
C_{min}=1- \frac 1 {g^2d}. \nonumber
\end{eqnarray}
\end{theorem}
\begin{proof}
We have that
\begin{eqnarray}
||\mathbf x_{j+1}-\mathbf x^*||^2 &=& ||\mathbf x_j^q-\mathbf x^*-td[\triangledown f(\mathbf x_j^q)]_{s_{j+1}}||^2 \nonumber \\
&=& ||\mathbf x_j^q-\mathbf x^*||^2+t^2d^2||[\triangledown f(\mathbf x_j^q)]_{s_{j+1}}||^2 \nonumber \\
&-&2td\langle [\triangledown f(\mathbf x_j^q)]_{s_{j+1}},\mathbf x_j^q-\mathbf x^* \rangle.
\end{eqnarray}
Taking the expectation of both sides with respect to the independent and identically distributed (i.i.d.) random variables $S_1,S_2, ...S_{j+1}$
\begin{eqnarray}
\mathbb E||\mathbf x_{j+1}-\mathbf x^*||^2 &=& \mathbb E||\mathbf x_j^q-\mathbf x^*||^2 \nonumber \\
&+&t^2d^2\mathbb E||[\triangledown f(\mathbf x_j^q)]_{s_{j+1}}||^2 \nonumber \\
&-&2td \mathbb E \langle [\triangledown f(\mathbf x_j^q)]_{s_{j+1}},\mathbf x_j^q-\mathbf x^* \rangle. \nonumber \\
\end{eqnarray}
Since $\mathbb E_{s_{j+1}}[\triangledown f(\mathbf x_j^q)]_{s_{j+1}}=\frac 1 d (\triangledown f(\mathbf x_j^q))$, then
\begin{eqnarray}
\mathbb E||\mathbf x_{j+1}-\mathbf x^*||^2 &=& \mathbb E||\mathbf x_j^q-\mathbf x^*||^2+t^2d\mathbb E||\triangledown f(\mathbf x_j^q)||^2 \nonumber \\
&-&2t\mathbb E\langle \triangledown f(\mathbf x_j^q),\mathbf x_j^q-\mathbf x^* \rangle. \label{eq:expectation}
\end{eqnarray}
By applying inequalities (\ref{eq:smooth}) and (\ref{eq:convex}), and using the fact that $\triangledown f(\mathbf x^*)=0$, we have that
\begin{eqnarray}
||\triangledown f(\mathbf x_j^q)|| \le L ||\mathbf x_j^q-\mathbf x^*||, \label{eq:smooth1}
\end{eqnarray}
and
\begin{eqnarray}
\langle \triangledown f(\mathbf x_j^q),\mathbf x_j^q-\mathbf x^* \rangle \ge m ||\mathbf x_j^q-\mathbf x^*||. \label{eq:convex1}
\end{eqnarray}
Substituting (\ref{eq:smooth1}) and (\ref{eq:convex1}) in (\ref{eq:expectation}), we get that
\begin{eqnarray} \label{eq:main}
\mathbb E||\mathbf x_{j+1}-\mathbf x^*||^2 &\le& C \mathbb E||\mathbf x_j^q-\mathbf x^*||^2,
\end{eqnarray}
where $C=t^2L^2d-2tm+1$. We also have that
\begin{eqnarray}
||\mathbf x_j^q-\mathbf x^*||^2&=&||\mathbf x_j-\mathbf x^*+td\mathbf n_j||^2 \nonumber \\
&=& ||\mathbf x_j-\mathbf x^*||^2 + t^2d^2||\mathbf n_j||^2 \nonumber \\
&+& 2td \langle \mathbf x_j-\mathbf x^*,\mathbf n_j \rangle \nonumber \\
&\le& ||\mathbf x_j-\mathbf x^*||^2 + t^2d^2||\mathbf n_j||^2 \nonumber \\
&+& 2td ||\mathbf x_j-\mathbf x^*||||\mathbf n_j|| \nonumber \\
&\le& ||\mathbf x_j-\mathbf x^*||^2 + td\Delta ||\mathbf x_j-\mathbf x^*|| \nonumber \\
&+& \frac {t^2d^2\Delta^2} 4,
\end{eqnarray}
where the first inequality follows from Cauchy-Schwarz inequality, and the second inequality follows from (\ref{eq:quan}).

To proceed with the convergence analysis, we have two different cases.

{\bf Case 1} ($||\mathbf x_0-\mathbf x^*|| \le 1$):

In this case, $\mathbb E||\mathbf x_j-\mathbf x^*|| \le 1$. Therefore,
\begin{eqnarray}
\mathbb E||\mathbf x_{j+1}-\mathbf x^*||^2 \le C\mathbb E||\mathbf x_j-\mathbf x^*||^2 + Ctd\Delta (1+\frac {td\Delta} 4).
\end{eqnarray}
Let $k_1$ denotes the minimum number of iterations required to achieve the convergence condition. Hence,
\begin{eqnarray}
\mathbb E||\mathbf x_{k_1}-\mathbf x^*||^2 &\le& C^{k_1}||\mathbf x_0-\mathbf x^*||^2 \nonumber \\
&+&Ctd\Delta (1+\frac {td\Delta} 4)(1+C+..+C^{k_1-1}). \nonumber \\
\end{eqnarray}
Since $C < 1$, then
\begin{eqnarray}
\mathbb E||\mathbf x_{k_1}-\mathbf x^*||^2 &\le& C^{k_1}||\mathbf x_0-\mathbf x^*||^2 \nonumber \\
&+&\frac {C} {1-C}td\Delta (1+\frac {td\Delta} 4).
\end{eqnarray}
For the algorithm to converge, let
\begin{eqnarray} \label{eq:iter1}
C^{k_1}||\mathbf x_0-\mathbf x^*||^2 \le \frac {\epsilon \rho} 2,
\end{eqnarray}
and
\begin{eqnarray} \label{eq:cond1}
\frac {C} {1-C}td\Delta (1+\frac {td\Delta} 4) \le \frac {\epsilon \rho} 2,
\end{eqnarray}

{\bf Case 2} ($||\mathbf x_0-\mathbf x^*|| > 1$):

Let $k_2$ denotes the minimum number of iterations required such that $\mathbb E||\mathbf x_{k_2}-\mathbf x^*|| \le 1$. For all $j \le k_2$, we have that $\mathbb E||\mathbf x_j-\mathbf x^*|| \le \mathbb E||\mathbf x_j-\mathbf x^*||^2$. Therefore,
\begin{eqnarray}
\mathbb E||\mathbf x_{j+1}-\mathbf x^*||^2 &\le& C(1+td\Delta)\mathbb E||\mathbf x_j-\mathbf x^*||^2 \nonumber \\
&+& \frac {C t^2d^2\Delta^2} 4.
\end{eqnarray}
After $k_2$ iterations, we have that
\begin{eqnarray}
\mathbb E||\mathbf x_{k_2}-\mathbf x^*||^2 &\le& (C(1+td\Delta))^{k_2}||\mathbf x_0-\mathbf x^*||^2 \nonumber \\
&+&\frac {C t^2d^2\Delta^2} {4 (1-C)}.
\end{eqnarray}
For the algorithm to converge, let
\begin{eqnarray} \label{eq:iter2}
(C(1+td\Delta))^{k_2}||\mathbf x_0-\mathbf x^*||^2 \le \frac 1 2,
\end{eqnarray}
and
\begin{eqnarray} \label{eq:cond2}
\frac {C t^2d^2\Delta^2} {4 (1-C)} \le \frac 1 2.
\end{eqnarray}
Finally, the total number of iterations required for convergence is given by
\begin{eqnarray}
k^q=k_1+k_2.
\end{eqnarray}

To achieve the fastest convergence rate, the step size $t$ is chosen to minimize $C$. Hence,
\begin{eqnarray}
t_{opt}=\frac 1 {gLd}~~,\text{and}~~C_{min}=1- \frac 1 {g^2d}
\end{eqnarray}

From (\ref{eq:cond1}) and (\ref{eq:cond2}), a sufficient condition on the quantization error is given by
\begin{eqnarray} \label{eq:cond3}
\Delta \le \frac {\epsilon \rho L^2} {2m}(\frac 1 C_{min}-1).
\end{eqnarray}

From (\ref{eq:iter1}), (\ref{eq:iter2}), and (\ref{eq:cond3}), the number of iterations required for the algorithm to converge is at most
\begin{eqnarray} \label{eq:iter3}
k^q&=& \frac {\log (2||\mathbf x_0-\mathbf x^*||^2/\epsilon \rho)}{\log(1/C_{min})} \nonumber \\
&+& \frac {\log (2||\mathbf x_0-\mathbf x^*||^2)}{\log(1/(C_{min}+\frac {\epsilon \rho}{2}(1-C_{min}))}.
\end{eqnarray}
\end{proof}

Note that By setting $\Delta=0$ and hence $x_j^q=x_j$ in (\ref{eq:main}), the quantization-free scenario can be recovered. It follows that the number of iterations required for the quantization-free algorithm to converge is at most
\begin{eqnarray} \label{eq:qfree}
k&=& \frac {\log (||\mathbf x_0-\mathbf x^*||^2/\epsilon \rho)}{\log(1/C_{min})},
\end{eqnarray}
which coincides with the result obtained in \cite{jordan}.

\section{Simulation Results} \label{sec:simulation}
In this section, we run an experiment to verify that the quantization error does not propagate and hence the convergence is possible. For that purpose, we apply the quantized randomized coordinate descent algorithm to solve a linear regression problem. The data set we use is collected from a power plant \cite{pinar}. It has four predictors (Temperature, Pressure, Humidity, and Exhaust Vacuum) and one output (Electrical Energy). All data is normalized to have zero mean and a standard deviation of one. The number of observations is $n=9568$. 

To solve this problem, it is required to minimize the square loss function
\begin{eqnarray}
f(\mathbf x)=\frac 1 2 \sum_{i=1}^n (y_i-A_{i:}\mathbf x)^2,
\end{eqnarray}
where $A$ is the data matrix, $A_{i:}$ is the $i$th row of $A$, and $\mathbf y$ is the output vector. Notice that the first column of $A$ is a vector of ones, which is added to evaluate for the intercept. The network consists of five nodes in addition to the fusion center; the first node calculates the derivative in the direction of the intercept coefficient, while each of the remaining nodes calculates the derivative in the direction of one predictor coefficient. The algorithm starts from $\mathbf x_0=[1,1,1,1,1]^T$ and iterates to reduce the coefficients residual $||\mathbf x_j-\mathbf x^*||^2$.

{\bf Experiment 1:} $t=10^{-4}$, $\Delta=10^5$.

First, we plot the coefficients residual against the number of iterations as shown in Fig.~\ref{Fig-1}.

\begin{figure}[htb]
\centering
\includegraphics[width=0.52\textwidth]{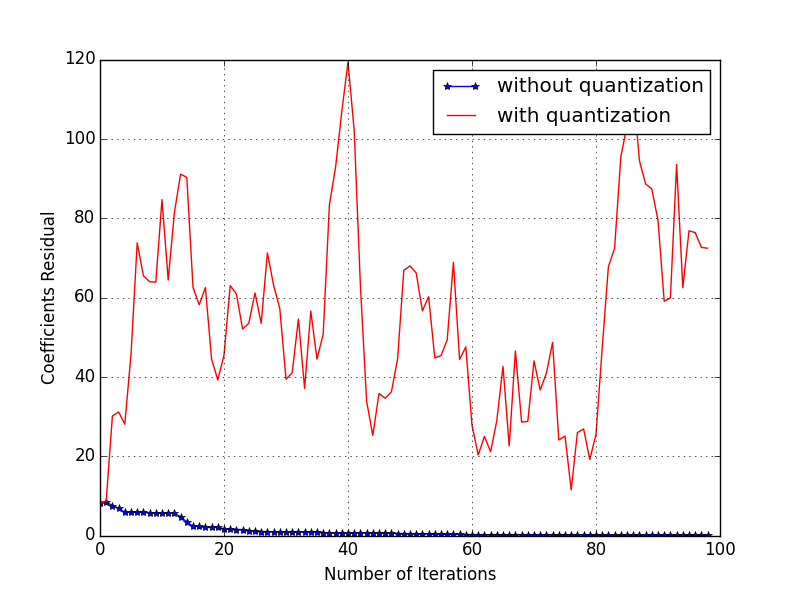}\\
\caption{Effect of the quantization error ($\Delta=10^5$) on the coefficients residual.}
\label{Fig-1}
\end{figure}

Then, we plot the predicted value for an input of all ones $A_1=[1,1,1,1,1]$ against the number of iterations as shown in Fig.~\ref{Fig-2}.

\begin{figure}[htb]
\centering
\includegraphics[width=0.52\textwidth]{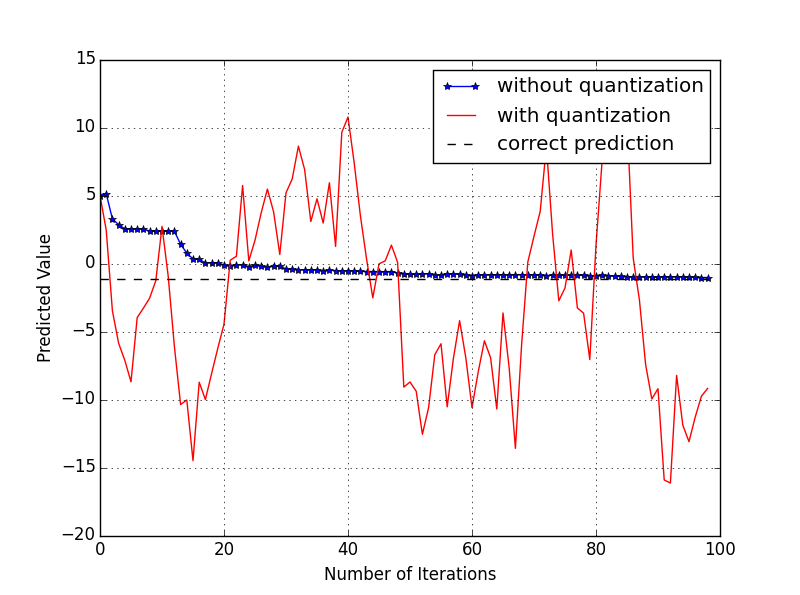}\\
\caption{Effect of the quantization error ($\Delta=10^5$) on the predicted value.}
\label{Fig-2}
\end{figure}

Figures~\ref{Fig-1} and \ref{Fig-2} show that the quantized randomized coordinate descent algorithm diverges if the quantization error $\Delta=10^5$. This result is intuitive since a large quantization error is expected to prevent the algorithm from converging to the optimal solution.

{\bf Experiment 2:} $t=10^{-4}$, $\Delta=10^3$.

First, we plot the coefficients residual against the number of iterations as shown in Fig.~\ref{Fig-3}.

\begin{figure}[htb]
\centering
\includegraphics[width=0.52\textwidth]{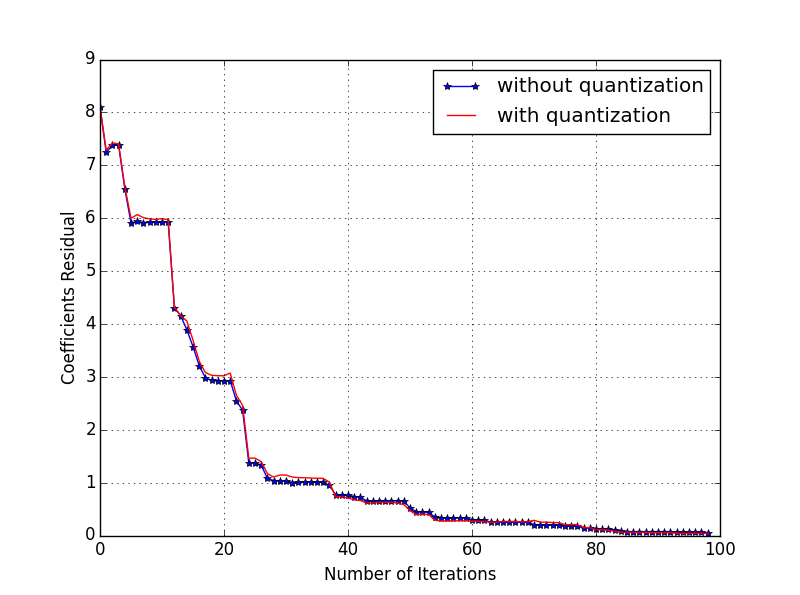}\\
\caption{Effect of the quantization error ($\Delta=10^3$) on the coefficients residual.}
\label{Fig-3}
\end{figure}

Then, we plot the predicted value for an input of all ones $A_1=[1,1,1,1,1]$ against the number of iterations as shown in Fig.~\ref{Fig-4}.

\begin{figure}[htb]
\centering
\includegraphics[width=0.52\textwidth]{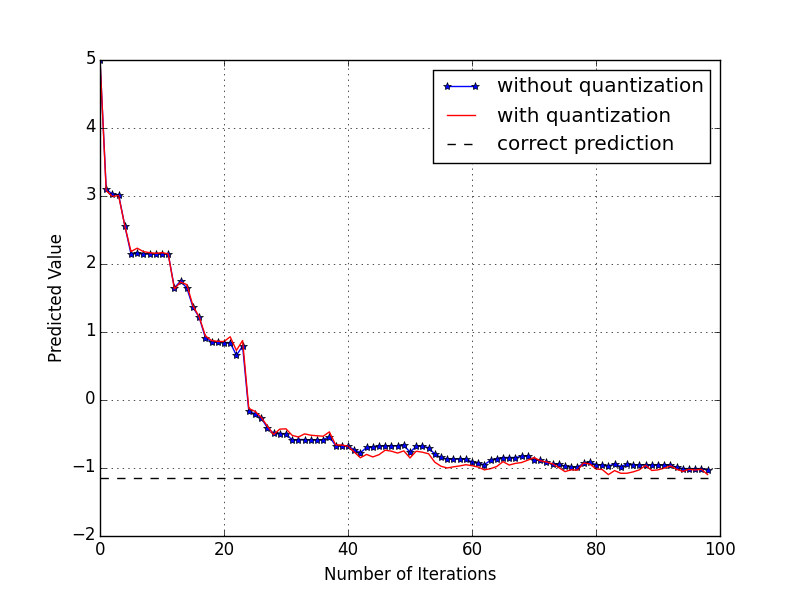}\\
\caption{Effect of the quantization error ($\Delta=10^3$) on the predicted value.}
\label{Fig-4}
\end{figure}

Figures~\ref{Fig-3} and \ref{Fig-4} show that the quantized randomized coordinate descent algorithm converges for a smaller value of the quantization error $\Delta=10^3$. This verifies that the quantization error does not propagate and hence the convergence is possible if the quantization error is bounded, which coincides with the result we obtained in Theorem 1.

\section{Conclusion} \label{sec:conclusion}
In this paper, we have studied the problem of distributed optimization under communication constraints. We have modified the randomized coordinate descent algorithm to solve an unconstrained convex minimization problem in the presence of quantization error. We have analyzed the convergence rate of our algorithm, and we have derived sufficient conditions on the quantization error to guarantee that the algorithm converges to the optimal solution. We have further verified that the convergence is possible in the presence of quantization error by running an experiment that solves a linear regression problem.

\bibliographystyle{ieeetr}
\bibliography{macros,icasp}

\begin{thebibliography}{10}

\bibitem{martin}
M.~Jaggi, V.~Smith, M.~Tak{\'a}\v{c}, J.~Terhorst, S.~Krishnan, T.~Hofmann, and
  M.~I. Jordan, {\em Coordinate descent algorithms}.
\newblock In Advances in NIPS 27, pp. 3068--3076, 2014.

\bibitem{yuchen}
Y.~Zhang and L.~Xiao, {\em DiSCO: Distributed optimization for self-concordant
  empirical loss}.
\newblock In ICML, pp. 362--370, 2015.

\bibitem{boyd}
S.~Boyd, N.~Parikh, E.~Chu, B.~Peleato, and J.~Eckstein, {\em Distributed
  optimization and statistical learning via the alternating direction method of
  multipliers}.
\newblock Foundations and Trends in Machine Learning, vol. 3, pp. 1--122, 2010.

\bibitem{wei}
W.~Deng and W.~Yin, {\em On the global and linear convergence of the
  generalized alternating direction method of multipliers}.
\newblock Journal of Scientific Computing, pp. 1--28, 2012.

\bibitem{hong}
S.~Zhu, M.~Hong, and B.~Chen, ``Quantized consensus admm for multi-agent
  distributed optimization,'' in {\em Proc. IEEE Intl. Conf. on Acoustics,
  Speech, and Signal Processing}, (Shanghai, China), pp.~4134--4138, Mar. 2016.

\bibitem{wright}
S.~J. Wright, {\em Coordinate descent algorithms}.
\newblock Mathematical Programming, vol. 151, no. 1, pp. 3--34, 2015.

\bibitem{peter1}
P.~Richt{\'a}rik and M.~Tak{\'a}\v{c}, {\em Iteration complexity of randomized
  block-coordinate descent methods for minimizing a composite function}.
\newblock Mathematical Programming, pages 1--38, 2012.

\bibitem{peter2}
R.~Tappenden, P.~Richt{\'a}rik, and J.~Gondzio, ``Inexact coordinate descent:
  complexity and preconditioning,'' {\em Journal of Optimization Theory and
  Applications}, vol.~170, pp.~144--176, Jul. 2016.

\bibitem{peter3}
P.~Richt{\'a}rik and M.~Tak{\'a}\v{c}, {\em Distributed coordinate descent
  method for learning with big data}.
\newblock arXiv:1310.2059, 2013.

\bibitem{jordan}
H.~Mania, X.~Pan, D.~Papailiopoulos, B.~Recht, K.~Ramchandran, and M.~I.
  Jordan, {\em Perturbed iterate analysis for asynchronous stochastic
  optimization}.
\newblock arXiv:1507.06970v2, 2016.

\bibitem{pinar}
P.~Tufekci, {\em Prediction of full load electrical power output of a base load
  operated combined cycle power plant using machine learning methods}.
\newblock International Journal of Electrical Power and Energy Systems, vol.
  60, pp. 126--140, 2014.

\end{thebibliography}

\end{document}